\documentclass[runningheads]{llncs}
\usepackage{graphicx}
\usepackage{tikz}
\usepackage{comment}
\usepackage{amsmath,amssymb} 
\usepackage{color}

\usepackage{epsfig}
\usepackage{adjustbox}
\usepackage{kotex}
\usepackage{flushend}

\usepackage{amsthm}

\usepackage[resetlabels,labeled]{multibib}
\newcites{New}{References}

\usepackage{soul}
\newcommand\ml[1]{{\color{black}#1}}
\newcommand\nj[1]{{\color{black}#1}}
\newcommand\sh[1]{{\color{black}#1}}

\theoremstyle{plain}

\newtheorem{lem}{Lemma}

\theoremstyle{definition}
\newtheorem{defn}{Definition}

\theoremstyle{remark}

\usepackage{xspace}
\makeatletter
\DeclareRobustCommand\onedot{\futurelet\@let@token\@onedot}
\def\@onedot{\ifx\@let@token.\else.\null\fi\xspace}

\def\eg{\emph{e.g}\onedot} 
\def\ie{\emph{i.e}\onedot}

\def\wrt{w.r.t\onedot} 
\def\etal{\emph{et al}\onedot}
\makeatother

\makeatletter
\newcommand{\printfnsymbol}[1]{%
  \textsuperscript{\@fnsymbol{#1}}%
}
\makeatother
\DeclareMathOperator*{\argmin}{argmin}

\begin{document}
\pagestyle{headings}
\mainmatter
\def\ECCVSubNumber{6446}  

\title{Procrustean Regression Networks: Learning 3D Structure of Non-Rigid Objects from 2D Annotations} 

\titlerunning{Procrustean Regression Networks}
%
\author{Sungheon Park\thanks{Authors contributed equally. $^{\dagger}$ Corresponding author.}\inst{1}\orcidID{0000-0002-7287-5661} \and
Minsik Lee\printfnsymbol{1}\inst{2}\orcidID{0000-0003-4941-4311} \and
Nojun Kwak$^{\dagger}$\inst{3}\orcidID{0000-0002-1792-0327}}

\authorrunning{S. Park et al.}
%
\institute{Samsung Advanced Institute of Technology (SAIT), Korea \email{sungheonpark@snu.ac.kr} \and
Hanyang University, Korea
\email{mleepaper@hanyang.ac.kr} \and
Seoul National University, Korea
\email{nojunk@snu.ac.kr}}
\maketitle

\begin{abstract}
 We propose a novel framework for training neural networks which is capable of learning 3D information of non-rigid objects when only 2D annotations are available as ground truths. Recently, there have been some approaches that incorporate the problem setting of non-rigid structure-from-motion (NRSfM) into deep learning to learn 3D structure reconstruction. The most important difficulty of NRSfM is to estimate both the rotation and deformation at the same time, and previous works handle this by regressing both of them. In this paper, we resolve this difficulty by proposing a loss function wherein the suitable rotation is automatically determined. Trained with the cost function consisting of the reprojection error and the low-rank term of aligned shapes, the network learns the 3D structures of such objects as human skeletons and faces during the training, whereas the testing is done in a single-frame basis. The proposed method can handle inputs with missing entries and experimental results validate that the proposed framework shows superior reconstruction performance to the state-of-the-art method on the Human 3.6M, 300-VW, and SURREAL datasets, even though the underlying network structure is very simple.
\end{abstract}

\section{Introduction}

\begin{figure*}[t]
    \centering
    \hfill
    \includegraphics[width=0.98\textwidth]{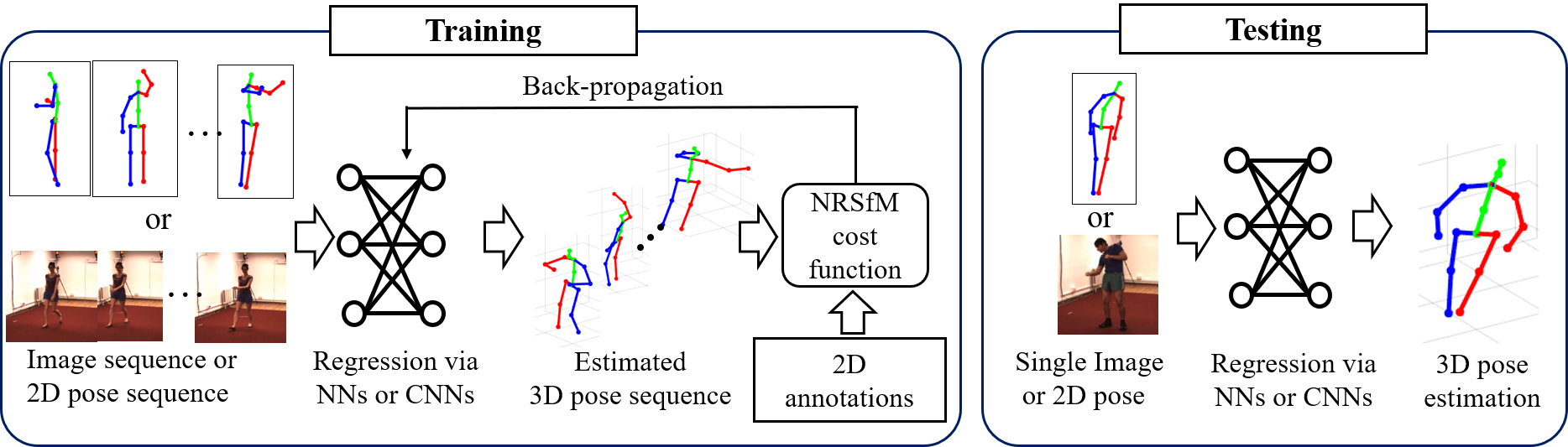}
    \hfill

\caption{Illustration of PRN. During the training, sequences of images or 2D poses are fed to the network, and their 3D shapes are estimated as the network outputs. The network is trained using the cost function which is based on an NRSfM algorithm. Testing is done by a simple feed-forward operation in a single-frame basis.}
\label{fig1}
\end{figure*}

Inferring 3D poses from several 2D observations is inherently an underconstrained problem. Especially, for non-rigid objects such as human faces or bodies, it is harder to retrieve the 3D shapes than for rigid objects due to their shape deformations.

There are two distinct ways to retrieve 3D shapes of non-rigid objects from 2D observations. The first approach is to use a 3D reconstruction algorithm. Non-rigid structure from motion (NRSfM) algorithms~\cite{akhter2011trajectory,bregler2000recovering,dai2014simple,gotardo2011non,lee2013procrustean} are designed to reconstruct 3D shapes of non-rigid objects from a sequence of 2D observations. Since NRSfM algorithms are not based on any learned models, the algorithms should be applied to each individual sequence, which makes the algorithm time-consuming when there are numerous number of sequences. The second approach is to learn the mappings from 2D to 3D with 3D ground truth training data. Prior knowledge can be obtained by dictionary learning~\cite{zhou20153d,Zhou_2016_CVPR}, but neural networks or convolutional neural networks (CNNs) are the most-used methods to learn the 2D-to-3D or image-to-3D mappings~\cite{martinez2017simple,pavlakos2017coarse}, recently. However, 3D ground truth data are essential to learn those mappings, which requires large amounts of costs and efforts compared to the 2D data acquisition.

There is another possibility: With the framework which combines those two different frameworks, \ie, NRSfM and neural networks, it is possible to overcome the limitations and to take advantages of both. There have been a couple of works that implement NRSfM using deep neural networks~\cite{kong2019deep,cha2019unsupervised}, but these methods mostly focus on the structure-from-category (SfC) problem, in which the 3D shapes of different rigid subjects in a category are reconstructed, and the deformation between subjects are not very diverse. Experiments on the CMU MoCap data in \cite{kong2019deep} show that, for data with diverse deformations, its generalization performance is not very good. Recently, Novotny \etal~\cite{novotny2019c3dpo} proposed a neural network that reconstructs 3D shapes from monocular images by canonicalizing 3D shapes so that the 3D rigid motion is registered. This method has shown successful reconstruction results for data with more diverse deformations, which has been used in traditional NRSfM research. Wang \etal~\cite{wang2019distill} also proposed knowledge distillation method that incorporate NRSfM algorithms as a teacher, which showed promising results on learning 3D human poses from 2D points.

The main difficulty of NRSfM is that one has to estimate both the rigid motion and the non-rigid shape deformation, which has been discussed extensively in the field of NRSfM throughout the past two decades. Especially, motion and deformation can get mixed up and some parts of rigid motions can be mistaken to be deformations. This has been first pointed out in \cite{lee2013procrustean}, in which conditions derived from the generalized Procrustes analysis (GPA) has been adopted to resolve the problem. Meanwhile, all recent neural-network-based NRSfM approaches attempt to regress both the rigid motion and non-rigid deformation at the same time. Among these, only Novotny \etal~\cite{novotny2019c3dpo} deals with the motion-deformation-separation problem in NRSfM, which is addressed as ``shape transversality.'' Their solution is to register motions of different frames using an auxiliary neural network.


In this paper, we propose an alternative to this problem: First, we prove that a set of Procrustes-aligned shapes is transversal. Based on this fact, rather than explicitly estimating rigid motions, we propose a novel loss, in which suitable motions are determined automatically based on Procrustes alignment. This is achieved by modifying the cost function recently proposed in Procrustean regression (PR)~\cite{PE_TIP}, an NRSfM scheme that shares similar motivations with our work, which is used to train neural networks via back-propagation. Thanks to this new loss function, the network can concentrate only on the 3D shape estimation, and accordingly, the underlying structure of the proposed neural network is quite simple. The proposed framework, \textit{Procrustean Regression Network} (PRN), learns to infer 3D structures of deformable objects using only 2D ground truths as training data.

Figure~\ref{fig1} illustrates the flow of the proposed framework. PRN accepts a set of image sequences or 2D point sequences as inputs at the training phase. The cost function of PRN is formulated to minimize the reprojection error and the nuclear norm of aligned shapes.
The whole training procedure is done in an end-to-end manner, and the reconstruction result for an individual image is generated at the test phase via a simple forward propagation without requiring any post processing step for 3D reconstruction. Unlike the conventional NRSfM algorithms, PRN robustly estimates 3D structure of unseen test data with feed-forward operations in the test phase, taking the advantage of neural networks. The experimental results verify that PRN effectively reconstructs the 3D shapes of non-rigid objects such as human faces and bodies.



\section{Related Works}

The underlying assumption of NRSfM methods is that the 3D shape or the 3D trajectory of a point is interpreted as a weighted sum of several bases~\cite{bregler2000recovering,akhter2011trajectory}. 3D shapes are obtained by factorizing a shape matrix or a trajectory matrix so that the matrix has a pre-defined rank. Improvements have been made by several works which use probabilistic principal components analysis~\cite{torresani2008nonrigid}, metric constraints~\cite{paladini2009factorization}, course-to-fine reconstruction algorithm~\cite{bartoli2008coarse}, complementary-space modeling~\cite{gotardo2011non}, block sparse dictionary learning~\cite{kong2016prior}, or force-based models~\cite{agudo2018force}. The major disadvantage of early NRSfM methods is that the number of basis should be determined explicitly while the optimal number of bases is usually unknown and is different from sequence to sequence. NRSfM methods using low-rank optimization have been proposed to overcome this problem~\cite{dai2014simple,garg2013dense}.

It was proven that shape alignment also helps to increase the performance of NRSfM~\cite{lee2013procrustean,lee2014procrustean,lee2016pami,cho2016complex}. Procrustean normal distribution (PND)~\cite{lee2013procrustean} is a powerful framework to separate rigid shape variations from the non-rigid ones. 
The expectation-maximization-based optimization algorithm applied to PND, EM-PND, showed superior performance to other NRSfM algorithms. Based on this idea, Procrustean Regression (PR)~\cite{PE_TIP} has been proposed to optimize an NRSfM cost function via a simple gradient descent method. In~\cite{PE_TIP}, the cost function consists of a data term and a regularization term where low-rankness is imposed not directly on the reconstructed 3D shapes but on the aligned shapes with respect to the reference shape. Any type of differentiable function can be applied for both terms, which has allowed its applicability to perspective NRSfM. 




On the other hand, along with recent rise of deep learning, there have been efforts to solve 3D reconstruction problems using CNNs. Object reconstruction from a single image with CNNs is an active field of research. The densely reconstructed shapes are often represented as 3D voxels or depth maps. While some works use ground truth 3D shapes~\cite{Tatarchenko2016,choy20163d,wu2017marrnet}, other works enable the networks to learn 3D reconstruction from multiple 2D observations~\cite{yan2016perspective,tulsiani2017multi,gadelha20163d,zhang2017learning}. The networks used in aforementioned works include a transformation layer that estimates the viewpoint of observations and/or a reprojection layer to minimize the error between input images and projected images. However, they mostly restrict the class of objects to ones that are rigid and have small amounts of deformations within each class, such as chairs and tables.

The 3D interpreter network~\cite{Wu2016} took a similar approach to NRSfM methods in that it formulates 3D shapes as the weighted sum of base shapes, but it used synthetic 3D models for network training. Warpnet~\cite{Kanazawa_2016_CVPR} successfully reconstructs 3D shapes of non-rigid objects without supervision, but the results are only provided for birds datasets which have smaller deformations than human skeletons. Tulsiani \etal~\cite{tulsiani2017learning} provided a learning algorithm that automatically localize and reconstruct deformable 3D objects, and Kanazawa \etal~\cite{cmrKanazawa18} also infer 3D shapes as well as texture information from a single image. Although those methods output dense 3D meshes, the reconstruction is conducted on rigid objects or birds which do not contain large deformations. Our method provides a way to learn 3D structure of non-rigid objects that contain relatively large deformations and pose variations such as human skeletons or faces.

Training a neural network using the loss function based on NRSfM algorithms has been rarely studied. Kong and Lucey \cite{kong2019deep} proposed to interpret NRSfM as multi-layer sparse coding, and Cha \etal \cite{cha2019unsupervised} proposed to estimate multiple basis shapes and rotations from 2D observations based on a deep neural network. However, they mostly focused on solving SfC problems which have rather small deformations, and the generalization performance of Kong and Lucey \cite{kong2019deep} is not very good for unseen data with large deformations.  Recently, Novotny \etal~\cite{novotny2019c3dpo} proposed a network structure which factors object deformation and viewpoint changes. Even though many existing ideas in NRSfM are nicely implemented in \cite{novotny2019c3dpo}, this in turn makes the network structure quite complicated. Unlike \cite{novotny2019c3dpo}, the 3D shapes are aligned to the mean of aligned shapes in each minibatch in PRN, which enables the use of a simple network structure. Moreover, PRN does not need to set the number of basis shapes explicitly, because it is adjusted automatically in the low-rank loss.



\section{Method}
We briefly review PR~\cite{PE_TIP} in Section~\ref{sec:PR}, which is a regression problem based on Procrustes-aligned shapes and is the basis of PRN. Here, we also introduce the concept of ``shape transversality'' proposed by Novotny \etal~\cite{novotny2019c3dpo} and prove that a set of Procrustes-aligned shapes is transversal, which means that Procrustes alignment can determine unique motions and eliminate the rigid motion components from reconstructed shapes. The cost function of PRN and its derivatives are explained in Section~\ref{sec:PRN}. The data term and the regularization term for PRN are proposed in Section~\ref{sec:cost}. Lastly, network structures and training strategy is described in Section~\ref{sec:detail}.

\subsection{Procrustean Regression}
\label{sec:PR}

NRSfM aims to recover 3D positions of the deformable objects from 2D correspondences. Concretely, given 2D observations of $n_p$ points $\mathbf{U}_i (1 \leq i \leq n_f)$ in $n_f$ frames, NRSfM reconstructs 3D shapes of each frame $\mathbf{X}_i$. PR~\cite{PE_TIP} formulated NRSfM as a regression problem. The cost function of PR consists of data term that corresponds to the reprojection error and the regularization term that minimizes the rank of the aligned 3D shapes, which has the following form:
\begin{equation}\label{eq1}
  \mathcal{J} = \sum_{i=1}^{n_{f}}f(\mathbf{X_\mathnormal{i}}) + \lambda g(\mathbf{\widetilde{X}},\overline{\mathbf{X}}).
\end{equation}
Here, $\mathbf{X}_{i}$ is a $3\times n_p$ matrix of the reconstructed 3D shapes on the $i$th frame, and $\overline{\mathbf{X}}$ is a reference shape for Procrustes alignment. $\mathbf{\widetilde{X}}$ is a $3n_p \times n_f$ matrix which is defined as $\widetilde{\mathbf{X}} \triangleq [\mathrm{vec}({\widetilde{\mathbf{X}}_1}) \, \mathrm{vec}({\widetilde{\mathbf{X}}_2}) \, \cdots \, \mathrm{vec}({\widetilde{\mathbf{X}}_{n_f}})]$, where $\mathrm{vec}(\cdot)$ is a vectorization operator. $\widetilde{\mathbf{X}}_i$ is an aligned shape of the $i$th frame. The aligned shapes are retrieved via Procrustes analysis without scale alignment. In other words, the aligning rotation matrix for each frame is calculated as
 \begin{equation}\label{eq2}
   \mathbf{R}_i = \argmin_{\mathbf{R}}{\lVert \mathbf{R} \mathbf{X}_i \mathbf{T} - 
   \overline{\mathbf{X}} \rVert} \quad \mathrm{s.t.} \quad \mathbf{R}^\emph{T}\mathbf{R} = \mathbf{I}.
 \end{equation}
Here, $\mathbf{T} \triangleq \mathbf{I}_{n_p}-\frac{1}{n_p}\mathbf{1}_{n_p}\mathbf{1}_{n_p}^\emph{T}$ is the translation matrix that makes the shape centered at origin. $\mathbf{I}_{n}$ is an $n \times n$ identity matrix, and $\mathbf{1}_{n}$ is an all-one vector of size $n$. The aligned shape of the $i$th frame becomes $\tilde{\mathbf{X}}_i = \mathbf{R}_i \mathbf{X}_i \mathbf{T}$.

In~\cite{PE_TIP}, (\ref{eq1}) is optimized for variables $\mathbf{X}_{i}$ and $\overline{\mathbf{X}}$ and it is shown that their gradients for (\ref{eq1}) can be analytically derived. Hence, any gradient-based optimization method can be applied for large choices of $f$ and $g$. What the above formulation implies is that we can impose a regularization loss based on the alignment of reconstructed shapes, and therefore, we can enforce certain properties only to non-rigid deformations in which rigid motions are excluded.

To back up the above claim, we introduce the transversal property introduced in \cite{novotny2019c3dpo}:
\begin{defn}
The set $\mathcal{X}_0 \subset \mathbb{R}^{3 \times n_p}$ has the transversal property if, for any pair $\mathbf{X}, \mathbf{X}' \in \mathcal{X}_0$ related by a rotation $\mathbf{X}' = \mathbf{R} \mathbf{X}$, then $\mathbf{X} = \mathbf{X}'$.
\end{defn}
The above definition basically defines a set of shapes that do not contain any non-trivial rigid transforms of its elements, and its elements can be interpreted as having canonical rigid poses. In other words, if two shapes in the set are distinctive, then they should not be identical up to a rigid transform. Here, we prove that the set of Procrustes-aligned shapes is indeed a transversal set. First, we need an assumption: Each shape should have a unique Procrustes alignment \wrt the reference shape. This condition might not be satisfied in some cases, \eg, degenerate shapes such as co-linear shapes.
\begin{lem}
A set $\mathcal{X}_P$ of Procrustes-aligned shapes \wrt a reference shape $\overline{\mathbf{X}}$ is transversal if the shapes are not degenerate.
\end{lem}
\begin{proof}
Suppose that there are $\mathbf{X}, \mathbf{X}' \in \mathcal{X}_P$ that satisfy $\mathbf{X}' = \mathbf{R} \mathbf{X}$. Based on the assumption, $\min_{\mathbf{R}'} \| \mathbf{R}' \mathbf{X}' \mathbf{T} - \overline{\mathbf{X}} \|^2$ will have a unique minimum at $\mathbf{R}' = \mathbf{I}$. Hence, $\min_{\mathbf{R}'} \| \mathbf{R}' \mathbf{R} \mathbf{X} \mathbf{T} - \overline{\mathbf{X}} \|^2$ will also have a unique minimum at the same point, which indicates that $\min_{\mathbf{R}''} \| \mathbf{R}'' \mathbf{X} \mathbf{T} - \overline{\mathbf{X}} \|^2$ will have one at $\mathbf{R}'' = \mathbf{R}$. Based on the assumption, $\mathbf{R}''$ has to be $\mathbf{I}$, and hence $\mathbf{R}=\mathbf{I}$.
\end{proof}
In \cite{novotny2019c3dpo}, an arbitrary registration function $f$ is introduced to ensure the transversality of a given set, which is implemented as an auxiliary neural network that has to be trained together with the main network component. We can interpret the Procrustes alignment in this work as a replacement of $f$ that does not need training and has analytic gradients. Accordingly, the underlying network structure of PRN can become much simpler at the cost of a more complicated loss function.


\subsection{PR Loss for Neural Networks}
\label{sec:PRN}

One may directly use the gradients of (\ref{eq1}) to train neural networks by designing a neural network that estimates both the 3D shapes $\mathbf{X}_{i}$ and the reference shape $\overline{\mathbf{X}}$.
However, the reference shape here incurs some problems when we are to handle it in a neural network. If the class of objects that we are interested in does not contain large deformations, then imposing this reference shape as a global parameter can be an option. On the contrary, if there can be a large deformation, then optimizing the cost function with minibatches of similar shapes or sequences of shapes can be vital for the success of training. In this case, a separate network module to estimate a good 3D reference shape is inevitable. However, designing a network module that estimates mean shapes may make the network structure more complex and training procedure harder.
To keep it concise, we excluded the reference shape from (\ref{eq1}) and defined the reference shape as the mean of the aligned output 3D shapes. The mean shape $\overline{\mathbf{X}}$ in (\ref{eq1}) is simply replaced with $\sum_{j=1}^{n_f} \mathbf{R}_j \mathbf{X}_j \mathbf{T}$. Now, $\mathbf{X}_i$ is the only variable in the cost function, and the derivative of the cost function with respect to the estimated 3D shapes, $\frac{\partial {\mathcal{J}}}{\partial{\mathbf{X}_i}}$, is derived analytically.

The cost function of PRN can be written as follows:
\begin{equation}\label{eq3}
  \mathcal{J} = \sum_{i=1}^{n_{f}}f(\mathbf{X_\mathnormal{i}}) + \lambda g(\mathbf{\widetilde{X}}).
\end{equation}
The alignment constraint is also changed to
 \begin{equation}\label{eq4}
   \mathbf{R} = \argmin_{\mathbf{R}}{ \sum_{i=1}^{n_f} \lVert \mathbf{R}_i \mathbf{X}_i \mathbf{T} - \frac{1}{n_f}
   \sum_{j=1}^{n_f} \mathbf{R}_j \mathbf{X}_j \mathbf{T} \rVert}  \quad\quad \mathrm{s.t.} \quad \mathbf{R}_{i}^\emph{T}\mathbf{R}_{i} = \mathbf{I}.
 \end{equation}
where $\mathbf{R}$ is the concatenation of all rotation matrices, \ie, $\mathbf{R}=[\mathbf{R}_1, \mathbf{R}_2, \cdots ,\mathbf{R}_{n_f}]$. Let us define $\mathbf{X}$ and $\widetilde{\mathbf{X}}$ as ${\mathbf{X}} \triangleq [\mathrm{vec}(\mathbf{X}_{1}), \mathrm{vec}(\mathbf{X}_{2}), \cdots, \mathrm{vec}(\mathbf{X}_{n_f})]$ and $\widetilde{\mathbf{X}} \triangleq [\mathrm{vec}(\widetilde{\mathbf{X}}_{1}),$ $\mathrm{vec}(\widetilde{\mathbf{X}}_{2}), \cdots, \mathrm{vec}(\widetilde{\mathbf{X}}_{n_f})]$ respectively. The gradient of $\mathcal{J}$ with respect to $\mathbf{X}$ while satisfying the constraint (\ref{eq4}) is
\begin{equation}\label{eq5}
\frac{\partial {\mathcal{J}}}{\partial{\mathbf{X}}} = \frac{\partial {f}}{\partial{\mathbf{X}}} + \lambda \left\langle \frac{\partial {g}}{\partial{\widetilde{\mathbf{X}}}}, \frac{\partial {\widetilde{\mathbf{X}}}}{\partial{\mathbf{X}}} \right\rangle,
\end{equation}
where $\left\langle \cdot , \cdot \right\rangle$ denotes the inner product. $\frac{\partial {f}}{\partial{\mathbf{X}}}$ and $\frac{\partial {g}} {\partial{\widetilde{\mathbf{X}}}}$ are derived once $f$ and $g$ are determined. The derivation process of $\frac{\partial {\widetilde{\mathbf{X}}}}{\partial{\mathbf{X}}}$ is analogous to~\cite{PE_TIP}. We explained detailed process in the supplementary material and provide only the results here, which has the form of
\begin{equation}\label{eq6}
\frac{\partial {\widetilde{\mathbf{X}}}}{\partial{\mathbf{X}}} = (\mathbf{A}\mathbf{B}^{-1}\mathbf{C}+\mathbf{I}_{3 n_p n_f})\mathbf{D}.
\end{equation}
$\mathbf{A}$ is a $3 n_p n_f \times 3 n_f$ block diagonal matrix expressed as
\begin{equation}\label{eq7}
\mathbf{A} = \text{blkdiag}((\mathbf{X}_{1}^{\prime T} \otimes \mathbf{I}_{3})\mathbf{L}, (\mathbf{X}_{2}^{\prime T} \otimes \mathbf{I}_{3})\mathbf{L}, \cdots, (\mathbf{X}_{n_f}^{\prime T} \otimes \mathbf{I}_{3})\mathbf{L}),
\end{equation}
where $\text{blkdiag}(\cdot)$ is the block-diagonal operator, $\otimes$ denotes the Kronecker product. $\mathbf{X}_{i}^{\prime T} = \mathbf{\hat{R}}_i \mathbf{X}_i \mathbf{T}$, where $\mathbf{\hat{R}}_i$ is the current rotation matrix before the gradient evaluation, and $\mathbf{L}$ is a $9 \times 3$ matrix that implies the orthogonality constraint of a rotation matrix~\cite{PE_TIP}, whose values are
\begin{equation}\label{eq8}
  \mathbf{L} =
    \begin{bmatrix}
0 & 0 & 0 & 0 & 0 & -1 & 0 & 1 & 0 \\
0 & 0 & 1 & 0 & 0 & 0 & -1 & 0 & 0 \\
0 & -1 & 0 & 1 & 0 & 0 & 0 & 0 & 0 \\
\end{bmatrix}^\emph{T}.
\end{equation}
$\mathbf{B}$ is a $3n_f \times 3n_f$ matrix whose block elements are
\begin{equation}\label{eq9}
        \mathbf{b}_{ij}=
        \begin{cases}
            \mathbf{L}^\emph{T} ( \sum_{k \neq i} \mathbf{X}_{k}^{\prime T}\mathbf{X}_{i}^{\prime T} \otimes \mathbf{I}_{3} ) \mathbf{L} & i=j \\
            \mathbf{L}^\emph{T} ( \mathbf{I}_{3} \otimes \mathbf{X}_{i}^{\prime} \mathbf{X}_{j}^{\prime T} ) \mathbf{EL} & i \neq j
        \end{cases}
\end{equation}
where $\mathbf{b}_{ij}$ means the $(i, j)$-th $3 \times 3$ submatrix of $\mathbf{B}$, $i$ and $j$ are integers ranging from $1$ to $n_f$, and $\mathbf{E}$ is a permutation matrix that satisfies $\mathbf{E}\mathrm{vec}(\mathbf{H}) = \mathrm{vec}(\mathbf{H}^{T})$.
$\mathbf{C}$ is a $3n_f \times 3 n_f n_p$ matrix whose block elements are
\begin{equation}\label{eq10}
        \mathbf{c}_{ij}=
        \begin{cases}
            - \mathbf{L}^\emph{T} ( \sum_{k \neq i} \mathbf{X}_{k}^{\prime} \otimes \mathbf{I}_{3} ) & i=j \\
            - \mathbf{L}^\emph{T} ( \mathbf{I}_{3} \otimes \mathbf{X}_{i}^{\prime} ) \mathbf{E} & i \neq j
        \end{cases}
\end{equation}
where $\mathbf{c}_{ij}$ means the $(i, j)$-th $3 \times 3$ submatrix of $\mathbf{C}$. Finally, $\mathbf{D}$ is a $3 n_f n_p \times 3 n_f n_p$ block-diagonal matrix expressed as
\begin{equation}\label{eq11}
\mathbf{D} = \text{blkdiag}(\mathbf{T} \otimes \mathbf{\hat{R}}_1, \mathbf{T} \otimes \mathbf{\hat{R}}_2, \cdots, \mathbf{T} \otimes \mathbf{\hat{R}}_{n_f}).
\end{equation}

Even though the size of $\partial {\widetilde{\mathbf{X}}}/\partial{\mathbf{X}}$ is quite large, \ie, $3 n_f n_p \times 3 n_f n_p$, we don't actually have to construct it explicitly since the only thing we need is the ability to backpropagate. Memory space and computations can be largely saved based on clever utilization of batch matrix multiplications and reshapes. In the next section, we will discuss about the design of the functions $f$ and $g$ and their derivatives.


\subsection{Design of $f$ and $g$}
\label{sec:cost}

In PRN, the network produces the 3D position of each joint of a human body. The network output is fed into the cost function, and the gradients are calculated to update the network. For the data term $f$, we use the reprojection error between the estimated 3D shapes and the ground truth 2D points. We only consider the orthographic projection in this paper, but the framework can be easily extended to the perspective projection. The function $f$ corresponding to the data term has the following form.
\begin{equation}\label{eq12}
  f(\mathbf{X}) = \sum_{i=1}^{n_f} \frac{1}{2}{\lVert (\mathbf{U}_i - \mathbf{P}_{o} \mathbf{X}_i)\odot \mathbf{W}_i \rVert}_{F}^{2}.
\end{equation}
Here, $\mathbf{P}_{o} = \big[\begin{smallmatrix}
  1 & 0 & 0\\
  0 & 1 & 0
\end{smallmatrix}\big]$ is an $2 \times 3$ orthographic projection matrix, and $\mathbf{U}_i$ is a $2\times n_p$ 2D observation matrix (ground truth). $\mathbf{W}_i$ is a $2 \times n_p$ weight matrix whose $i$th column represents the confidence of the position of $i$th point. $\mathbf{W}_i$ has values between 0 and 1, where 0 means the keypoint is not observable due to occlusion. Scores from 2D keypoint detectors can be used as values of $\mathbf{W}_i$. Lastly, ${\lVert \cdot \rVert}_{F}$ and $\odot$ denotes the Frobenius norm and element-wise multiplication respectively. The gradient of (\ref{eq12}) is
\begin{equation}\label{eq13}
  \frac{\partial f}{\partial \mathbf{X}} = \sum_{i=1}^{n_f} \mathbf{P}_{o}^T ((\mathbf{P}_{o}\mathbf{X}_i - \mathbf{U}_i)\odot \mathbf{W}_i\odot \mathbf{W}_i).
\end{equation}
For the regularization term, we imposed a low-rank constraint to the aligned shapes. Log-determinant or the nuclear norm are two widely used functions and we choose the nuclear norm, \ie,
\begin{equation}\label{eq14}
  g(\mathbf{\widetilde{X}}) = {\lVert \mathbf{\widetilde{X}} \rVert}_{*},
\end{equation}
where ${\lVert \cdot \rVert}_{*}$ stands for the nuclear norm of a matrix. The subgradient of a nuclear norm can be calculated as
\begin{equation}\label{eq15}
  \frac{\partial g}{\partial \mathbf{\widetilde{X}}} = \mathbf{U}\mathrm{sign}(\mathbf{\Sigma})\mathbf{V}^T,
\end{equation}
where $\mathbf{U} \mathbf{\Sigma} \mathbf{V}^T$ is the singular value decomposition of $\mathbf{\widetilde{X}}$ and $\mathrm{sign}(\cdot)$ is the sign function. Note that the sign function is to deal with zero singular values. $\partial g/\partial \mathbf{\widetilde{X}}_i$ is easily obtained by reordering $\partial g/\partial \mathbf{\widetilde{X}}$.

\subsection{Network Structure}
\label{sec:detail}

\begin{figure*}[t]
    \centering
    \includegraphics[width=0.98\textwidth]{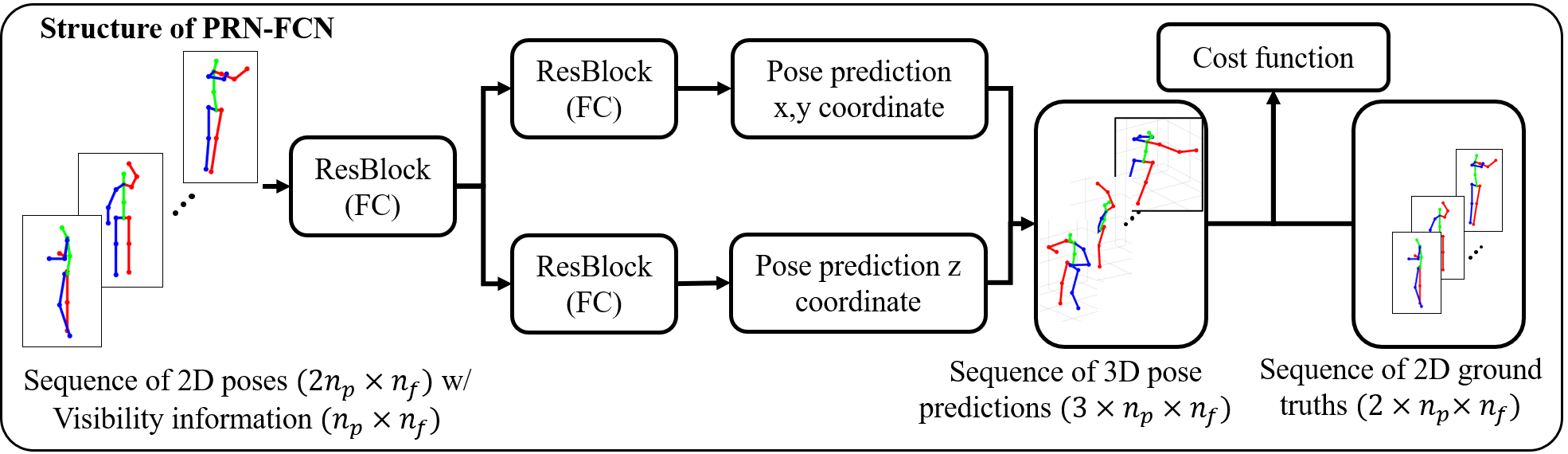}
\caption{Structure of FCNs used in this paper. Structure of CNNs are the same except that ResNet-50 is used as a backbone network instead of a fully connected ResBlock.}
\label{fig2}
\end{figure*}


By susbstituing (\ref{eq6}), (\ref{eq13}), and (\ref{eq15}) into (\ref{eq5}), the gradient of the cost function of PRN with respect to the 3D shape $\mathbf{X}_i$ can be calculated.
Then, the gradient for the entire parameters in the network can also be calculated by back-propagation. We experimented two different structures of PRN in Section~\ref{sec:exp}: fully connected networks (FCNs) and convolutional neural networks (CNNs). For the FCN structure, inputs are the 2D point sequences. Each minibatch has a size of $2 n_p \times n_f$, and the network produces the 3D positions of the input sequences. We use two stacks of residual modules~\cite{he2016deep} as the network structure. The prediction parts of x,y coordinates and z coordinates in the network are separated as illustrated in Figure~\ref{fig2}, which achieved better performance in our empirical experience. 

For the CNNs, sequences of RGB images are fed into the networks. ResNet-50~\cite{he2016deep} is used as a backbone network. The features of the final convolutional layers consisting of 2,048 feature maps of $7 \times 7$ size are connected to a network with the same structure as in the previous FCN to produce the final 3D output. We initialize the weights in the convolutional layers to those of the ImageNet~\cite{russakovsky2015imagenet} pre-trained network. More detailed hyperparameter settings are described in the supplementary material.


\section{Experiments}
\label{sec:exp}

The proposed framework is applied to reconstruct 3D human poses, 3D human faces, and dense 3D human meshes, all of which are the representative types of non-rigid objects. 
Additional qualitative results and experiments of PRN including comparison with the other methods on the datasets can be found in the supplementary materials.

\subsection{Datasets}
\textbf{Human 3.6M~\cite{ionescu2014human3}} contains large-scale action sequences with the ground truth of 3D human poses. We downsampled the frame rate of all sequences to 10 frames per second (fps). Following the previous works on the dataset, we used five subjects (S1, S5, S6, S7, S8) for training, and two subjects (S9, S11) are used as the test set. Both 2D points and RGB images are used for experiments. For the experiments with 2D points, we used ground truth projections provided in the dataset as well as the detection results of a stacked hourglass network~\cite{newell2016stacked}. 

\textbf{300VW~\cite{shen2015first}} has 114 video clips of faces with 68 landmarks annotations. We used the subset of 64 sequences from the dataset. The dataset is splitted into train and test sets, each of which consists of 32 sequences. 63,205 training images and 60,216 test images are used for the experiment. Since 300-VW dataset only provides 2D annotations and no 3D ground truth data exists, we used the data provided in~\cite{bulat2017far} as 3D ground truths.

\textbf{SURREAL~\cite{varol17_surreal}} dataset is used to validate our framework on dense 3D human shapes. It contains 3D human meshes which are created by fitting SMPL body model~\cite{SMPL} on CMU Mocap sequences. Each mesh is comprised of 6,890 vertices. We selected 25 sequences from the dataset and split into training and test sets which consist of 5,000 and 2,401 samples respectively. The meshes are randomly rotated around y-axis, and orthographic projection is applied to generate 2D points.

\subsection{Implementation details}
\label{sec:impl}
The parameter $\lambda$ is set to $\lambda=0.05$ for all experiments. The datasets used for our experiment consist of video sequences from fixed monocular cameras. However, most NRSfM algorithms including PRN requires moderate rotation variations in a sequence. To this end, for the Human 3.6M dataset where the sequences are taken by 4 different cameras, we alternately sample the frames or 2D poses from different cameras for consecutive frames. We set the time interval of the samples from different cameras to 0.5 seconds. Meanwhile, 300-VW dataset does not have multi-view sequences, and each sequence does not have enough rotations. Hence, we randomly sample the inputs in a minibatch from different sequences. For SURREAL datasets, we used 2D poses from consecutive frames.


On the other hand, the rotation alignment is applied to the samples within the same mini-batch. Therefore, if we select the samples in a mini-batch from a single sequence, the samples within the mini-batch does not have enough variations which affects training speed and performance. To alleviate this problem, we divided a mini-batch into 4 groups and calculated the gradients of the cost function for each group during the training of Human 3.6M and SURREAL datasets. In addition, since only a small number of different sequences are used in each mini-batch during the training and frames in the same mini-batch are highly correlated as a result, batch normalization~\cite{ioffe2015batch} may make the training unstable. Hence, we train the networks using batch normalization with moving average for the $70\%$ of the training, and the rest of the iterations are trained with fixed average values as in the test phase.



\subsection{Results}

\begin{table}[t]
\caption{MPJPE with 2D inputs on Human 3.6M dataset with different 2D inputs: (GT-ortho) Orthographic projection of 3D GT points. (GT-persp) Perspective projection of 3D GT. (SH)  2D keypoint detection results of a stacked hourglass network either from~\cite{novotny2019c3dpo} (SH~\cite{novotny2019c3dpo}) or from the network fine-tuned on Human 3.6M(SH-FT). PRN-FCN-W used weighted reprojection error based on the keypoint detection score.}
\centering
\label{tab1}         
\begin{adjustbox}{width=\linewidth}
\begin{tabular}{lcccccccccccccccc}
\hline
Method(GT-ortho) & Direct. & Discuss & Eating & Greet & Phone & Pose & Purch. & Sitting & SitingD & Smoke & Photo & Wait & Walk & WalkD & WalkT & Avg \\
\hline
PRN w/o reg & 138.1 & 139.7 & 146.5 & 145.2 & 140 & 127.6 & 149.4 & 170.4 & 188.4 & 138.3 & 150.9 & 133.1 & 125.6 & 143.9 & 139.3 & 144.8 \\

CSF2~\cite{gotardo2011non} + NN & 87.2 & 90.1 & 96.1 & 95.9 & 102.9 & 92.1 & 99.3 & 129.8 & 136.7 & 99.5 & 120.1 & 95.2 & 90.8 & 102.4 & 89.2 & 101.6 \\

SPM~\cite{dai2014simple} + NN & 65.3 &	68.7 &	82.0 &	70.1 &	95.3 &	65.1 &	71.9 &	117.0 & 136.0 &	84.3 &	88.9 &	71.2 &	59.5 &	73.3 &	68.3 &	82.3 \\

C3DPO~\cite{novotny2019c3dpo} & \bf 56.1	& \bf 55.6	& 62.2	& \bf 66.4	& 63.2	& \bf 62.0	& \bf 62.9	& \bf 76.3	& \bf 85.8	& \bf 59.9	& 88.7	& \bf 63.3	& 71.1	& 70.7	& 72.3	& 67.8 \\

PRN-FCN & 65.3 &	58.2 &	\bf 60.5 &	73.8 &	\bf 60.7 &	71.5 &	64.6 &	79.8 & 90.2 &	60.3 &	\bf 81.2 &	67.1 &	\bf 54.4 &	\bf 61.2 &	\bf 65.6 &	\bf 66.7 \\
\hline
Method(GT-persp) & Direct. & Discuss & Eating & Greet & Phone & Pose & Purch. & Sitting & SitingD & Smoke & Photo & Wait & Walk & WalkD & WalkT & Avg \\
\hline

C3DPO~\cite{novotny2019c3dpo} & 96.8 & 85.7 & 85.8 & 107.1 & 86.0 & 96.8 & 93.9 & 94.9 & 96.7 & 86.0 & 124.3 & 90.7 & 95.2 & 93.4 & 101.3 & 95.6  \\

PRN-FCN & \textbf{93.1}	& \textbf{83.3}	& \textbf{76.2}	& \textbf{98.6}	& \textbf{78.8}	& \textbf{91.7}	& \textbf{81.4}	& \textbf{87.4}	& \textbf{91.6}	& \textbf{78.2}	& \textbf{104.3}	& \textbf{89.6}	& \textbf{83.0}	& \textbf{80.5}	& \textbf{95.3}	& \textbf{86.4} \\

\hline
Method(SH) & Direct. & Discuss & Eating & Greet & Phone & Pose & Purch. & Sitting & SitingD & Smoke & Photo & Wait & Walk & WalkD & WalkT & Avg \\ \hline
C3DPO~\cite{novotny2019c3dpo} & 131.1 & 137.4 & 125.2 & 146.4 & 143.2 & 141.4 & 137.3 & 141.4 & 163.8 & 136.2 & 161 & 143.4 & 145.9 & 153.2 & 168.6 & 145.0\\
PRN-FCN (SH~\cite{novotny2019c3dpo}) & 127.2	& 115.1	& 109.2	& 130.0	& 126.9	& 122.3	& 116.4	& 128.4	& 149.3	& 117.3	& 140.7	& 124.0	& 123.9	& 115.3	& 140.4	& 124.5 \\

PRN-FCN (SH FT) & \textbf{100.2} & 89.4 & 83.8 & 105.5 & 93.0 & 97.2 & 89.2 & 114.0 & 141.2 & 89.1 & \textbf{114.8} & 97.3 & \textbf{91.0} & 88.3 & 107.2 & 99.1\\
PRN-FCN-W (SH FT) & 100.3 & \textbf{88.8} & \textbf{82.8} & \textbf{105.2} & \textbf{91.4} & \textbf{96.7} & \textbf{88.1} & \textbf{102.1} & \textbf{113.2} & \textbf{87.4 }& 115.1 & \textbf{96.5} & 91.7 & \textbf{87.6} & \textbf{106.4 }& \textbf{95.9 }\\
\hline
\end{tabular}
\end{adjustbox}
\end{table}

The performance of PRN on Human3.6M is evaluated in terms of mean per joint position error (MPJPE) which is the widely used metric in the literature. Meanwhile, we used normalized error as the error metric of 300-VW and SURREAL datasets since the dataset does not provide absolute scales of 3D points. MPJPE and normalized error(NE) are defined as
\begin{equation}\label{eq16}
  \mathrm{MPJPE}(\hat{\mathbf{X}}_i , \mathbf{X}_i^{*})  = \frac{1}{n_p}\sum_{j=1}^{n_p}\lVert \hat{\mathbf{X}}_{ij} - \mathbf{X}_{ij}^{*} \rVert , \quad
    \mathrm{NE}(\hat{\mathbf{X}}_i , \mathbf{X}_i^{*}) = \frac{\lVert \hat{\mathbf{X}}_i - \mathbf{X}_i^{*} \rVert_{F}} {\lVert \mathbf{X}_i^{*} \rVert_{F}},
\end{equation}
where $\hat{\mathbf{X}}_i$ and $\mathbf{X}_i^{*}$ denote the reconstructed 3D shape and the ground truth 3D shape on the $i$th frame, respectively, and $\hat{\mathbf{X}}_{ij}$ and $\mathbf{X}_{ij}^{*}$ are the $j$th keypoint of $\hat{\mathbf{X}}_i$ and $\mathbf{X}_i^{*}$, respectively. Since orthographic projection has reflection ambiguity, we measure the error also for the reflected shapes and choose the shape that has \nj{a smaller} error.

To verify the effectiveness of PRN in the fully-connected network architecture (PRN-FCN), we first applied PRN to the task of 3D reconstruction given an input of 2D points. First, we trained PRN-FCN on the Human 3.6M dataset using \nj{either} ground truth 2D generated by orthographic projection or perspective projection (GT-ortho, GT-persp) \nj{or} keypoints detected using Stacked hourglass networks (SH). The \nj{detailed results for different actions} are illustrated in Table~\ref{tab1}.
\nj{For comparison, we also show the results of C3DPO from \cite{novotny2019c3dpo} under the same training setting.} As a baseline, we also provide the performance of FCN trained only on the reprojection error (PRN w/o reg). We also trained the neural nets using the 3D shapes reconstructed from existing NRSfM methods, CSF2~\cite{gotardo2011non} and SPM~\cite{dai2014simple} to compare our framework with NRSfM methods. We applied the NRSfM methods to each sequence with \ml{the} same strides and camera settings as done in training PRN. The trained networks also have the same structure as the one used for PRN.

Here, we can confirm that the regularization term helps estimating depth information more accurately and drops the error significantly. Moreover, PRN-FCN significantly outperforms the NRSfM methods and is also superior to the recently proposed work~\cite{novotny2019c3dpo} for both ground truth inputs and inputs from keypoint detectors, which proves the effectiveness of the alignment and the low-rank assumption for similar shapes.. While PRN-FCN is silghtly better than~\cite{novotny2019c3dpo} under orthographic projections, it largely outperforms~\cite{novotny2019c3dpo} when trained using 2D points with persepective projections, which indicates that PRN is also robust to the noisy data. The results from the neural networks trained with NRSfM tend to have large variations depending on the types of sequences. This is mainly because the label data comes from NRSfM methods does not show prominent reconstruction results, and this erroneous signal limits the performance of the network in difficult sequences. On the other hand, PRN-FCN robustly reconstruct 3D shapes across all sequences. More interestingly, when the scores of keypoint detectors are used as a weight(PRN-FCN-W), PRN showed improved performance. This result implies that PRN is also robust to inputs with structured missing points since occluded keypoints have lower scores. Although we did not provide the confidence information as input signals, lower weight in the cost function makes the keypoints with lower confidence rely more on the regularization term. As a consequence, PRN-FCN-W performs especially better on the sequences that have complex pose variations such as \textit{Sitting} or \textit{SittingDown}.

\begin{figure}[t]
\centering

\begin{minipage}[r]{0.16\textwidth}
\centering
2D inputs
\end{minipage}
\begin{minipage}[r]{0.16\textwidth}
\centering
PRN-FCN
\end{minipage}
\begin{minipage}[r]{0.16\textwidth}
\centering
GT
\end{minipage}
\begin{minipage}[r]{0.16\textwidth}
\centering
2D inputs
\end{minipage}
\begin{minipage}[r]{0.16\textwidth}
\centering
PRN-FCN
\end{minipage}
\begin{minipage}[r]{0.16\textwidth}
\centering
GT
\end{minipage}

\begin{minipage}[r]{0.16\textwidth}
\centering
\includegraphics[width=0.6\textwidth]{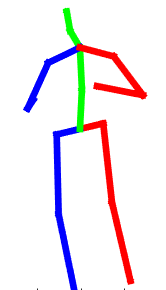}
\end{minipage}
\begin{minipage}[r]{0.16\textwidth}
\centering
\includegraphics[width=0.8\textwidth]{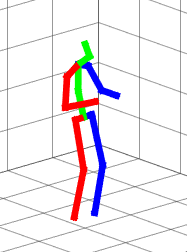}
\end{minipage}
\begin{minipage}[r]{0.16\textwidth}
\centering
\includegraphics[width=0.8\textwidth]{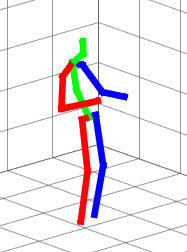}
\end{minipage}
\begin{minipage}[r]{0.16\textwidth}
\centering
\includegraphics[width=0.6\textwidth]{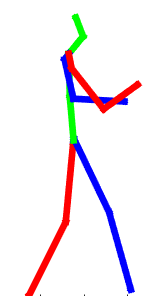}
\end{minipage}
\begin{minipage}[r]{0.16\textwidth}
\centering
\includegraphics[width=0.8\textwidth]{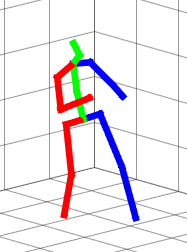}
\end{minipage}
\begin{minipage}[r]{0.16\textwidth}
\centering
\includegraphics[width=0.8\textwidth]{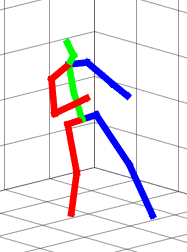}
\end{minipage}

\begin{minipage}[r]{0.16\textwidth}
\centering
\includegraphics[width=0.6\textwidth]{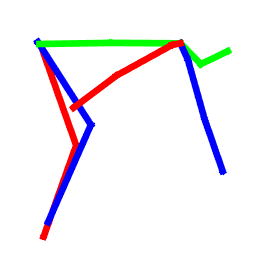}
\end{minipage}
\begin{minipage}[r]{0.16\textwidth}
\centering
\includegraphics[width=0.8\textwidth]{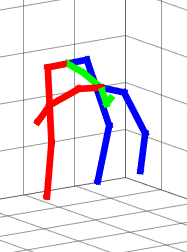}
\end{minipage}
\begin{minipage}[r]{0.16\textwidth}
\centering
\includegraphics[width=0.8\textwidth]{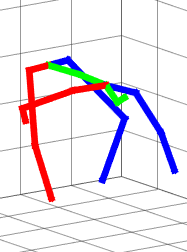}
\end{minipage}
\begin{minipage}[r]{0.16\textwidth}
\centering
\includegraphics[width=0.6\textwidth]{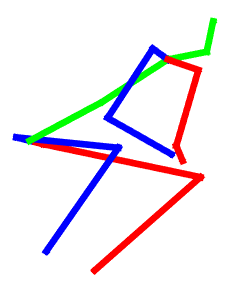}
\end{minipage}
\begin{minipage}[r]{0.16\textwidth}
\centering
\includegraphics[width=0.8\textwidth]{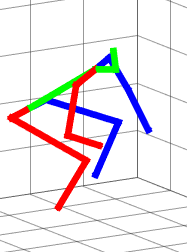}
\end{minipage}
\begin{minipage}[r]{0.16\textwidth}
\centering
\includegraphics[width=0.8\textwidth]{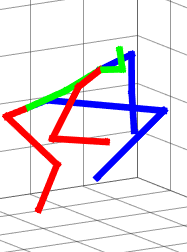}
\end{minipage}

  \caption{Qualitative results of PRN-FCN on Human 3.6M dataset. PRN successfully reconstructs 3D shapes from 2D points under various rotations and poses. Left arms and legs are shown in blue, and right arms and legs are shown in red.}
  \label{fig3}
\end{figure}

Qualitative results for PRN and comparison with the ground truth shapes are illustrated in Figure~\ref{fig3}. It is shown that PRN accurately reconstructs 3D shapes of human bodies from various challenging 2D poses.

\begin{table}[t]
\begin{minipage}{.35\linewidth}
    \centering

    \caption{MPJPE with RGB image inputs on the Human 3.6M dataset.}
    \label{tab2} 
    
    \medskip
\begin{tabular}{cc}
\hline
Method & \, MPJPE \,\\
\hline
PRN w/o reg & 164.5 
 \\
CSF2~\cite{gotardo2011non} + CNN & 130.6
 \\
SPM~\cite{dai2014simple} + CNN & 114.4
 \\
PRN-CNN & \bf 108.9
 \\
\hline
GT 3D & 98.8
\\
\hline
\end{tabular}
\end{minipage}\hfill
\begin{minipage}{.3\linewidth}
    \centering

    \caption{Normalized error with 2D inputs on the 300-VW dataset.}
    \label{tab3}

    \medskip

\begin{tabular}{cc}
\hline
Method & NE\\
\hline
PRN w/o reg & 0.5201
 \\
CSF2~\cite{gotardo2011non} + NN & 0.2751
 \\
PR~\cite{PE_TIP} + NN & 0.2730
 \\
C3DPO~\cite{novotny2019c3dpo} & 0.1715
 \\
PRN-FCN & \textbf{0.1512}
 \\
\hline
GT 3D & 0.0441
\\
\hline
\end{tabular}
\end{minipage}\hfill
\begin{minipage}{.3\linewidth}
    \centering

    \caption{Normalized error with 2D inputs on the SURREAL dataset.}
    \label{tab4}

    \medskip

\begin{tabular}{cc}
\hline
Method & NE\\
\hline
PRN w/o reg & 0.3565
 \\
C3DPO~\cite{novotny2019c3dpo} & 0.3509
 \\
PRN-FCN & \textbf{0.1377}
\\
\hline
\end{tabular}
\end{minipage}
\end{table}

Next, we apply PRN to the CNNs to learn the 3D shapes directly from RGB images. MPJPE on the Human 3.6M test set are provided in Table~\ref{tab2}. For comparison, we also trained the networks using only \ml{the} reprojection error \sh{and excluding the regularization term in the cost function of PRN (PRN w/o reg).}
Moreover, we also trained the networks using the 3D shapes reconstructed from existing NRSfM methods, CSF2~\cite{gotardo2011non} and PR~\cite{PE_TIP} since SPM~\cite{dai2014simple} diverged for many sequences in this dataset.  Estimating 3D poses from RGB images directly is more challenging than using 2D points as inputs because 2\nj{D} information as well as depth information should also be learned, and images also contain photometric variations or self-occlusions. PRN largely outperforms \sh{the model without regularization term} and shows better results than the CNNs trained using NRSfM reconstruction results. It can be observed that the CNN trained with ground truth 3D \ml{still} has large errors. The performance may be improved if recently\ml{-}proposed networks for 3D human pose estimation~\cite{mehta2017vnect,pavlakos2017coarse} is applied here. However, \ml{a} large network structure reduces \ml{the} batch size, which \ml{can ruin the entire training process of PRN.} 
Therefore, we \ml{instead used the largest network we can afford with maintaining the batch size to at least 16. Even though this limits the performance gain due to network structure, we can still compare the results from other similar-sized networks to verify that the proposed training strategy is effective.} 
Qualitative results of PRN-CNN are provided in the supplementary materials.

Next, for the task of 3D face reconstruction, we used \ml{the} 300-VW dataset~\cite{shen2015first} which has a video sequence of human faces. We used the reconstruction results from~\cite{bulat2017far} as 3D ground truths. The reconstruction performance is evaluated in terms of normalized error, and the results are illustrated in Table~\ref{tab3}. PRN-FCN is also superior to the other methods, including C3DPO~\cite{novotny2019c3dpo}, in 300-VW datasets. Qualitative results are shown in the two leftmost columns of Figure~\ref{fig4}. Both PRN and C3DPO output plausible results, but C3DPO tends to have larger depth ranges than ground truth depths, which led to increase the normalized errors.

Lastly, we validated the effectiveness of PRN on dense 3D models. Human meshes in SURREAL datasets consist of 6890 3D points for each shape. Since calculating the cost function on dense 3D data imposes heavy computational burden, we subdivided the 3D points into a few groups and compute the cost function for a small set of points. The groups are randomly organized in every iteration. Normalized errors on the SURREAL dataset is shown in Table~\ref{tab4}. As it can be seen in Table~\ref{tab4} and the two rightmost columns of Figure~\ref{fig4}, PRN-FCN effectively reconstruct 3D human mesh models from 2D inputs while C3DPO~\cite{novotny2019c3dpo} fails to recover depth information.

\begin{figure}[t]
\centering

\begin{minipage}[r]{0.2\textwidth}
\centering
 \qquad
\end{minipage}
\begin{minipage}[r]{0.35\textwidth}
\centering
300-VW dataset
\end{minipage}
\begin{minipage}[r]{0.4\textwidth}
\centering
SURREAL datset
\end{minipage}

\begin{minipage}[r]{0.18\textwidth}
\centering
2D inputs
\end{minipage}
\begin{minipage}[r]{0.18\textwidth}
\centering
\includegraphics[width=0.85\textwidth]{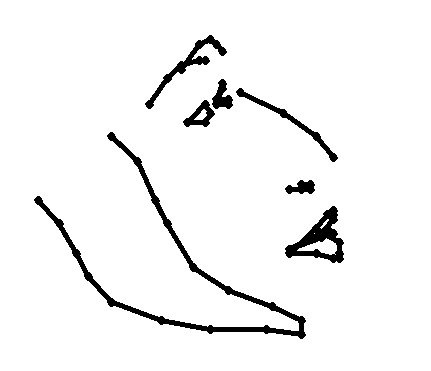}
\end{minipage}
\begin{minipage}[r]{0.18\textwidth}
\centering
\includegraphics[width=0.85\textwidth]{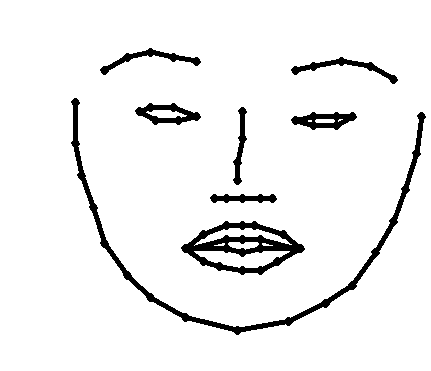}
\end{minipage}
\begin{minipage}[r]{0.18\textwidth}
\centering
\includegraphics[width=0.75\textwidth]{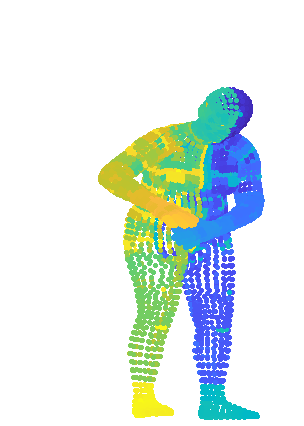}
\end{minipage}
\begin{minipage}[r]{0.18\textwidth}
\centering
\includegraphics[width=0.75\textwidth]{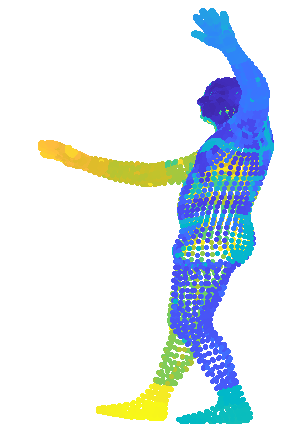}
\end{minipage}

\begin{minipage}[r]{0.18\textwidth}
\centering
C3DPO~\cite{novotny2019c3dpo}
\end{minipage}
\begin{minipage}[r]{0.18\textwidth}
\centering
\includegraphics[width=0.85\textwidth]{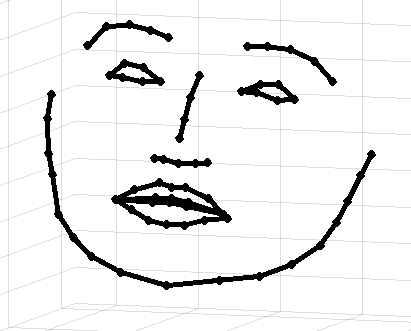}
\end{minipage}
\begin{minipage}[r]{0.18\textwidth}
\centering
\includegraphics[width=0.85\textwidth]{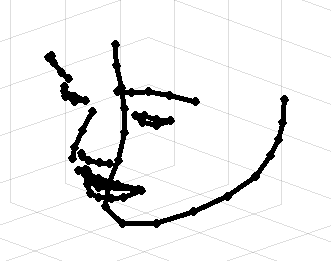}
\end{minipage}
\begin{minipage}[r]{0.18\textwidth}
\centering
\includegraphics[width=0.75\textwidth]{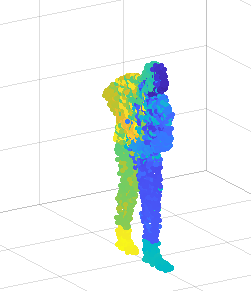}
\end{minipage}
\begin{minipage}[r]{0.18\textwidth}
\centering
\includegraphics[width=0.75\textwidth]{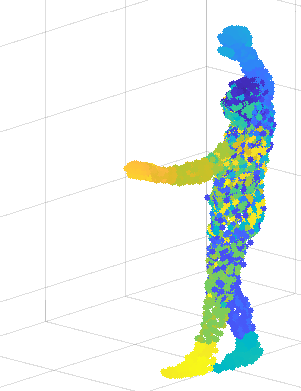}
\end{minipage}

\begin{minipage}[r]{0.18\textwidth}
\centering
PRN-FCN
\end{minipage}
\begin{minipage}[r]{0.18\textwidth}
\centering
\includegraphics[width=0.85\textwidth]{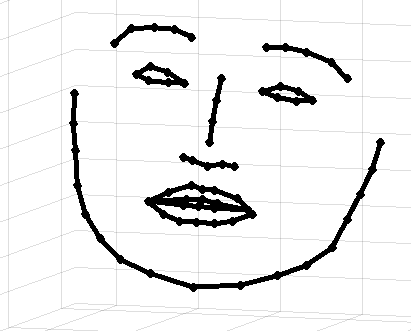}
\end{minipage}
\begin{minipage}[r]{0.18\textwidth}
\centering
\includegraphics[width=0.85\textwidth]{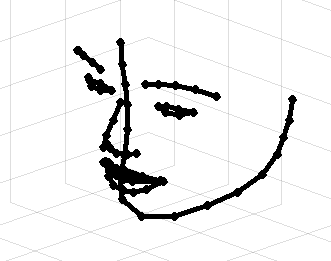}
\end{minipage}
\begin{minipage}[r]{0.18\textwidth}
\centering
\includegraphics[width=0.75\textwidth]{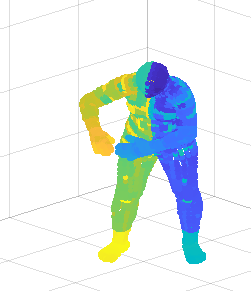}
\end{minipage}
\begin{minipage}[r]{0.18\textwidth}
\centering
\includegraphics[width=0.75\textwidth]{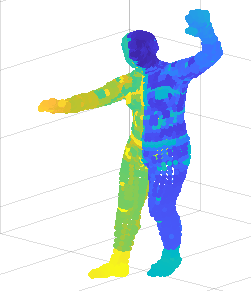}
\end{minipage}

\begin{minipage}[r]{0.18\textwidth}
\centering
GT
\end{minipage}
\begin{minipage}[r]{0.18\textwidth}
\centering
\includegraphics[width=0.85\textwidth]{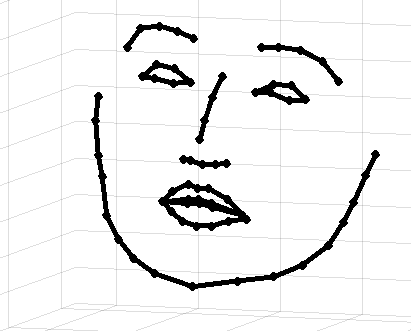}
\end{minipage}
\begin{minipage}[r]{0.18\textwidth}
\centering
\includegraphics[width=0.85\textwidth]{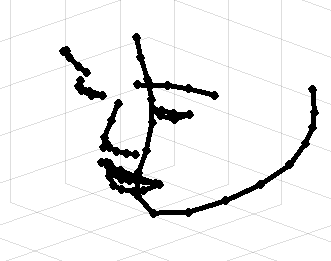}
\end{minipage}
\begin{minipage}[r]{0.18\textwidth}
\centering
\includegraphics[width=0.75\textwidth]{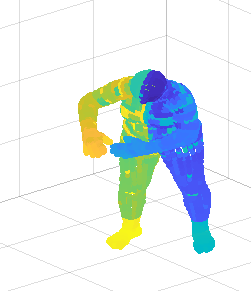}
\end{minipage}
\begin{minipage}[r]{0.18\textwidth}
\centering
\includegraphics[width=0.75\textwidth]{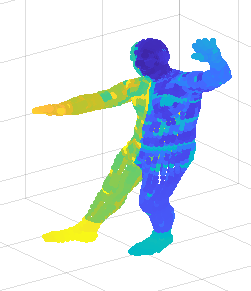}
\end{minipage}

  \caption{Qualitative results of PRN-FCN on 300-VW datasets (two leftmost columns) and SURREAL datasets (two rightmost columns).}
  \label{fig4}
\end{figure}

\section{Conclusion}

In this paper, a novel framework for training neural networks to estimate 3D shapes of non-rigid objects based on only 2D annotations \nj{is proposed}. 3D shapes of \nj{an image} can be rapidly estimated using the trained networks unlike \nj{existing} NRSfM algorithms. The performance of PRN can be improved by adopting different network architectures. For example, CNNs based on heatmap representations may provide accurate 2D poses and improve reconstruction performance. Moreover, the flexibility for designing \nj{the} data term and \nj{the} regularization term in PRN makes it easier to extend the framework to handle perspective projection. Nonetheless, the proposed PRN with simple network structures outperforms the existing state-of-the-art. Although \ml{solving NRSfM with deep learning still has some challenges}, we believe that the proposed framework establishes the connection between NRSfM algorithms and deep learning which will be useful for future research.

\section*{Acknowledgement}
This work was supported by grants from IITP (No.2019-0-01367, Babymind) and NRF Korea (2017M3C4A7077582, 2020R1C1C1012479), all of which are funded by the Korea government (MSIT).

\clearpage
%
%
\bibliographystyle{splncs04}
\bibliography{egbib}

\clearpage

\title{Procrustean Regression Networks: Learning 3D Structure of Non-Rigid Objects from 2D Annotations Supplementary Materials} 
\author{Sungheon Park\thanks{Authors contributed equally. $^{\dagger}$ Corresponding author.}\inst{1}\orcidID{0000-0002-7287-5661} \and
Minsik Lee\printfnsymbol{1}\inst{2}\orcidID{0000-0003-4941-4311} \and
Nojun Kwak$^{\dagger}$\inst{3}\orcidID{0000-0002-1792-0327}}

\authorrunning{S. Park et al.}
%
\institute{Samsung Advanced Institute of Technology (SAIT), Korea \email{sungheonpark@snu.ac.kr} \and
Hanyang University, Korea
\email{mleepaper@hanyang.ac.kr} \and
Seoul National University, Korea
\email{nojunk@snu.ac.kr}}
\maketitle

\section{Derivation of $\frac{\partial {\widetilde{\mathbf{X}}}}{\partial{\mathbf{X}}}$ in the Cost Function of PRN}

To make this material self-contained, we include the derivation of the back-propagation process of the proposed PRN, which is similar to that appeared in~\citeNew{PE_TIP_supp}. We start by rewriting the cost function of PRN:

\begin{equation}\label{eq_supp_1}
  \mathcal{J} = \sum_{i=1}^{n_{f}}f(\mathbf{X_\mathnormal{i}}) + \lambda g(\mathbf{\widetilde{X}}),
\end{equation}
while satisfying the alignment constraint
\begin{equation}\label{eq_supp_2}
   \mathbf{R} = \argmin_{\mathbf{R}}{ \sum_{i=1}^{n_f} \lVert \mathbf{R}_i \mathbf{X}_i \mathbf{T} -
   \frac{1}{n_f} \sum_{j=1}^{n_f} \mathbf{R}_j \mathbf{X}_j \mathbf{T} \rVert} \qquad \mathrm{s.t.} \quad \mathbf{R}_{i}^\emph{T}\mathbf{R}_{i} = \mathbf{I}_3.
\end{equation}
To integrate the alignment constraint (\ref{eq_supp_2}) with the cost function (\ref{eq_supp_1}), we introduce an orthogonal matrix $\mathbf{Q}_{i}$ that satisfies $\mathbf{R}_{i} = \mathbf{Q}_{i} \hat{\mathbf{R}}_{i}$ and assume $\mathbf{Q}_{i}=\mathbf{I}_3$  at the time of gradient evaluation without loss of generality. Then, $\widetilde{\mathbf{X}}_i = \mathbf{Q}_i \hat{\mathbf{R}}_i \mathbf{X}_i \mathbf{T} = \mathbf{Q}_i \mathbf{X}_i'$. Integrating the orthogonality constraint $\mathbf{Q}_{i}^{T} \mathbf{Q}_{i} = \mathbf{I}_3$ to (\ref{eq_supp_2}) by introducing Lagrange multipliers $\mathbf{\Lambda}_{i}$ yields the following equation:
\begin{equation}\label{eq_supp_3}
\sum_{i=1}^{n_f} \lVert \mathbf{Q}_{i} \mathbf{X}'_{i} - \frac{1}{n_f} \sum_{j=1}^{n_f} \mathbf{Q}_{j} \mathbf{X}'_{j}\rVert ^{2} +
 \frac{1}{2} \sum_{i=1}^{n_f} \left\langle \mathbf{\Lambda}_{i} , \mathbf{Q}_{i}^{T} \mathbf{Q}_{i} - \mathbf{I}_3 \right\rangle.
\end{equation}
Differentiating (\ref{eq_supp_3}) with respect to $\mathbf{Q}_{k} (1 \leq k \leq n_f)$ yields
\begin{equation}\label{eq_supp_4}
\begin{split}
(\mathbf{Q}_{k} \mathbf{X}'_{k} - \frac{1}{n_f} \sum_{j=1}^{n_f} \mathbf{Q}_{j} \mathbf{X}'_{j}) (\mathbf{X}_{k}'^{T} - \frac{1}{n_f}\mathbf{X}_{k}'^{T}) +
\sum_{i \neq k}(\mathbf{Q}_{i} \mathbf{X}'_{i} -  \frac{1}{n_f} \sum_{j=1}^{n_f} \mathbf{Q}_{j} \mathbf{X}'_{j})(-  \frac{1}{n_f} \mathbf{X}_{k}'^{T}) \\ + \mathbf{Q}_{k} \mathbf{\Lambda}_{k} = \mathbf{0}.
\end{split}
\end{equation}
By rearranging (\ref{eq_supp_4}) and multiplying $\mathbf{Q}_{k}^{T}$ on both sides, we get the following equation,
\begin{equation}\label{eq_supp_5}
\mathbf{Q}_{k} (\frac{n_f - 1}{n_f} \mathbf{X}'_{k} \mathbf{X}_{k}^{'T} + \mathbf{\Lambda}_{k}) \mathbf{Q}_{k}^{T}
= \frac{1}{n_f} \sum_{i \neq k} \mathbf{Q}_{i} \mathbf{X}'_{i} \mathbf{X}_{k}^{'T} \mathbf{Q}_{k}^{T}.
\end{equation}
On the left hand side, $\mathbf{\Lambda}_{k}$ is a symmetric matrix since the orthogonality constraint is symmetric (\ie, $\mathbf{Q}_{k}^{T} \mathbf{Q}_{k} = \mathbf{Q}_{k} \mathbf{Q}_{k}^{T} = \mathbf{I}_3$). Hence, the left hand side is symmetric, and so is the right hand side, \ie,
\begin{equation}\label{eq_supp_6}
    \sum_{i \neq k} \mathbf{Q}_{i} \mathbf{X}'_{i} \mathbf{X}_{k}^{'T} \mathbf{Q}_{k}^{T} = \sum_{i \neq k}  \mathbf{Q}_{k} \mathbf{X}'_{k} \mathbf{X}_{i}^{'T} \mathbf{Q}_{i}^{T} .
\end{equation}
By vectorizing (\ref{eq_supp_6}), we get
\begin{equation}\label{eq_supp_7}
\begin{split}
&\mathrm{vec}(\sum_{i \neq k}  \mathbf{Q}_{k} \mathbf{X}'_{k} \mathbf{X}_{i}^{'T} \mathbf{Q}_{i}^{T}) - \mathrm{vec}(\sum_{i \neq k} \mathbf{Q}_{i} \mathbf{X}'_{i} \mathbf{X}_{k}^{'T} \mathbf{Q}_{k}^{T}) \\
=&[(\sum_{i \neq k} \mathbf{Q}_{i} \mathbf{X}'_{i} \otimes \mathbf{I}_3)  - (\mathbf{I}_3 \otimes \sum_{i \neq k} \mathbf{Q}_{i} \mathbf{X}'_{i}) \mathbf{E}] \mathrm{vec}(\mathbf{Q}_{k} \mathbf{X}'_{k}) = \mathbf{0},
\end{split}
\end{equation}
where $\mathbf{E}$ is a permutation matrix that satisfies $\mathbf{E}\mathrm{vec}(\mathbf{H}) = \mathrm{vec}(\mathbf{H}^{T})$.
On the other hand, differentiating $\mathbf{Q}_{i}^{T} \mathbf{Q}_{i} = \mathbf{I}_3$ and evaluating at $\mathbf{Q}_{i} = \mathbf{I}_3$ gives
\begin{equation}\label{eq_supp_8}
  \partial \mathbf{Q}_i^\emph{T}\mathbf{Q}_i + \mathbf{Q}_i^\emph{T}\partial\mathbf{Q}_i = \partial \mathbf{Q}_i^\emph{T} + \partial\mathbf{Q}_i = 0.
\end{equation}
The above equation is a well-known relation about the derivative of an orthogonal matrix~\citeNew{hall2015lie}. Here, $\partial \mathbf{Q}_i$ is interpreted as an infinitesimal generator of rotation, which is a skew-symmetric matrix. Let us denote $\partial \mathbf{Q}_i$ as
\begin{equation}\label{eq_supp_9}
  \partial \mathbf{Q}_i =
  \begin{bmatrix}
0 & \partial q_{iz} & -\partial q_{iy} \\
-\partial q_{iz} & 0 & \partial q_{ix} \\
\partial q_{iy} & -\partial q_{ix} & 0 \\
\end{bmatrix},
\end{equation}
and if we define $\partial \mathbf{q}_i = [\partial q_{ix} \quad \partial q_{iy} \quad \partial q_{iz}]^\emph{T}$, then $\mathrm{vec}(\partial \mathbf{Q}_i) = \mathbf{L} \partial \mathbf{q}_i$.

Now, given an arbitrary $3 \times n_p$ matrix $\mathbf{S}$ and its column vectors $\mathbf{s}_1 , \mathbf{s}_2 , \cdots, \mathbf{s}_{n_p}$, one can easily verify that $(\mathbf{S}^\emph{T} \otimes \mathbf{I}_3)\mathbf{L} = \Big[ [\mathbf{s}_1]_\times^\emph{T} \quad [\mathbf{s}_2]_\times^\emph{T} \cdots [\mathbf{s}_{n_p}]_\times^\emph{T} \Big]^\emph{T}$ where $[\mathbf{s}]_\times$ is a skew-symmetric matrix that is related to the cross product of the vector. We can also obtain $[(\mathbf{S}^\emph{T} \otimes \mathbf{I}_3)\mathbf{L}]^{T}$ from $(\mathbf{S} \otimes \mathbf{I}_3)  - (\mathbf{I}_3 \otimes \mathbf{S}) \mathbf{E}$ by selecting 8th, 3rd, and 4th rows. The rest of the rows are either essentially identical to these rows or trivial. As a consequence, from (\ref{eq_supp_7}), we get
\begin{equation}\label{eq_supp_10}
  \mathbf{L}^{T} (\sum_{i \neq k} \mathbf{Q}_{i} \mathbf{X}'_{i} \otimes \mathbf{I}_3) \mathrm{vec}(\mathbf{Q}_{k} \mathbf{X}'_{k}) 
  = \mathbf{L}^{T} \mathrm{vec}(\mathbf{Q}_{k} \mathbf{X}'_{k} \sum_{i \neq k}  \mathbf{X}_{i}^{'T} \mathbf{Q}_{i}^{'T}) = \mathbf{0}.
\end{equation}
Differentiating (\ref{eq_supp_10}) yields
\begin{equation}\label{eq_supp_11}
    \begin{split}
        \mathbf{L}^{T} \mathrm{vec}(\partial \mathbf{Q}_k  \mathbf{X}'_{k} \sum_{i \neq k}  \mathbf{X}_{i}^{'T} \mathbf{Q}_{i}^{'T} + \partial \mathbf{X}'_{k} \sum_{i \neq k}  \mathbf{X}_{i}^{'T}
        + \mathbf{X}'_{k} \sum_{i \neq k} \partial \mathbf{X}_{i}^{'T}
        \\ + \mathbf{X}'_{k} \sum_{i \neq k}  \mathbf{X}_{i}^{'T} \partial \mathbf{Q}_{i}^{'T}
        )=\mathbf{0}.
    \end{split}
\end{equation}

By rearranging (\ref{eq_supp_11}) and substituting $\mathrm{vec}(\partial \mathbf{Q}_i) = \mathbf{L} \partial \mathbf{q}_i$ yields
\begin{equation}\label{eq_supp_12}
\begin{split}
&\mathbf{L}^{T} (\sum_{i \neq k} \mathbf{X}'_{i} \mathbf{X}_{k}^{'T} \otimes \mathbf{I}_3 ) \mathbf{L} \partial \mathbf{q}_k
+ \mathbf{L}^{T} \sum_{i \neq k} (\mathbf{I}_3 \otimes \mathbf{X}'_{k} \mathbf{X}_{i}^{'T})\mathbf{EL} \partial \mathbf{q}_i \\
&= - \mathbf{L}^{T}(\sum_{i \neq k} \mathbf{X}'_{i} \otimes \mathbf{I}_3) \mathrm{vec}(\partial \mathbf{X}'_{k})
- \mathbf{L}^{T} \sum_{i \neq k} (\mathbf{I}_3 \otimes \mathbf{X}'_{k}) \mathbf{E} \mathrm{vec}(\partial \mathbf{X}'_{i}).
\end{split}
\end{equation}
Since the index $k$ varies from $1$ to $n_f$, $n_f$ equations are made from (\ref{eq_supp_12}). Let $\partial \mathbf{q}$ be a vector $\partial \mathbf{q} = [\partial \mathbf{q}_1^{T}, \partial \mathbf{q}_2^{T}, \cdots, \partial \mathbf{q}_{n_f}^{T}]^{T}$, and similarly we define $\mathrm{vec} (\partial \mathbf{X}') = [\mathrm{vec} (\partial \mathbf{X}'_1)^{T}, \mathrm{vec} (\partial \mathbf{X}'_2)^{T}, \cdots, \mathrm{vec} (\partial \mathbf{X}'_{n_f})^{T}]^{T}$. To formulate $\partial \mathbf{q}$ as a function of $\mathrm{vec} (\partial \mathbf{X}')$, we enumerate $n_f$ equations and build a linear system that has the form of
\begin{equation}\label{eq_supp_13}
\mathbf{B} \partial \mathbf{q} = \mathbf{C} \mathrm{vec} (\partial \mathbf{X}').
\end{equation}
where $\mathbf{B}$ and $\mathbf{C}$ are the matrices explained in the main text.
Then, $\partial \mathbf{q}$ is expressed as
\begin{equation}\label{eq_supp_14}
\partial \mathbf{q} = \mathbf{B}^{-1} \mathbf{C} \mathrm{vec} (\partial \mathbf{X}').
\end{equation}

Next, differentiating $\mathbf{Q}_i \mathbf{X}_i' = \widetilde{\mathbf{X}}_i$ yields
\begin{equation}\label{eq_supp_15}
  \partial \mathbf{Q}_i \mathbf{X}_i' + \mathbf{Q}_i \partial \mathbf{X}_i' = \partial \widetilde{\mathbf{X}}_i.
\end{equation}
By vectorizing (\ref{eq_supp_15}), we get
\begin{equation}\label{eq_supp_16}
  ({\mathbf{X}_i'} ^\emph{T} \otimes \mathbf{I})\mathbf{L} \partial \mathbf{q}_i = \mathrm{vec}(\partial \widetilde{\mathbf{X}}_i - \partial \mathbf{X}_i').
\end{equation}
Let $\mathrm{vec} (\partial \widetilde{\mathbf{X}}) = [\mathrm{vec} (\partial \widetilde{\mathbf{X}}_1)^{T}, \mathrm{vec} (\partial \widetilde{\mathbf{X}}_2)^{T}, \cdots, \mathrm{vec} (\partial \widetilde{\mathbf{X}})_{n_f}^{T}]^{T}$, and similar to (\ref{eq_supp_13}), we build a linear system by varying the index $i$ from $1$ to $n_f$,
\begin{equation}\label{eq_supp_17}
\mathbf{A} \partial \mathbf{q} = \mathrm{vec}(\partial \widetilde{\mathbf{X}}) - \mathrm{vec}(\partial \mathbf{X}'),
\end{equation}
where $\mathbf{A}$ is also explained in the main text.

Substituting (\ref{eq_supp_14}) to (\ref{eq_supp_17}) yields
\begin{equation}\label{eq_supp_18}
(\mathbf{A} \mathbf{B}^{-1} \mathbf{C} +\mathbf{I}_{3 n_p n_f})  \mathrm{vec}(\partial \mathbf{X}') = \mathrm{vec}(\partial \widetilde{\mathbf{X}}).
\end{equation}
Finally, dividing both sides of (\ref{eq_supp_18}) by $\partial \mathrm{vec}(\mathbf{X})$ gives the derivative we need,
\begin{equation}\label{eq_supp_19}
\frac{\mathrm{vec} (\partial {\widetilde{\mathbf{X}}})}{\partial \mathrm{vec}(\mathbf{X})} = (\mathbf{A} \mathbf{B}^{-1} \mathbf{C}+\mathbf{I}_{3 n_p n_f})\mathbf{D},
\end{equation}
where $\mathbf{D}$ is a block-diagonal matrix explained in the main text. Note that $\mathbf{D}$ is based on the following derivative.
\begin{equation}\label{eq_supp_20}
\frac{\mathrm{vec}(\partial \mathbf{X}'_{i})}{\partial \mathrm{vec}(\mathbf{X}_i)} = \frac{ (\mathbf{T} \otimes \hat{\mathbf{R}_i}) \mathrm{vec}(\partial \mathbf{X}_{i})}{\partial \mathrm{vec}(\mathbf{X}_i)} = \mathbf{T} \otimes \hat{\mathbf{R}_i}.
\end{equation}

The derivative of the cost function is calculated as 
\begin{equation}\label{eq_supp_21}
\frac{\partial {\mathcal{J}}}{\partial{\mathrm{vec}(\mathbf{X})}} = \frac{\partial {f}}{\partial{\mathrm{vec}(\mathbf{X})}} + \lambda \left\langle \frac{\partial {g}}{\partial{\mathrm{vec}(\widetilde{\mathbf{X}})}}, \frac{\mathrm{vec} (\partial {\widetilde{\mathbf{X}}})}{\partial \mathrm{vec}(\mathbf{X})} \right\rangle,
\end{equation}
where the dimensions of $\frac{\partial {\mathcal{J}}}{\partial{\mathrm{vec}(\mathbf{X})}}$, $\frac{\partial {f}}{\partial{\mathrm{vec}(\mathbf{X})}}$, and $\frac{\partial {g}}{\partial{\mathrm{vec}(\widetilde{\mathbf{X}})}}$ are $1 \times 3n_p n_f$, and the dimension of $\frac{\mathrm{vec} (\partial {\widetilde{\mathbf{X}}})}{\partial \mathrm{vec}(\mathbf{X})}$ is $3n_p n_f \times 3n_p n_f$.

\section{Training Details}
For the Human 3.6M dataset, PRN-FCN with 2D ground truth inputs receives 17 keypoints for each frame while that with stacked hourglass network detection inputs receive 16 keypoints. Both networks produce 17 keypoints as an output which is trained with ground truth 2D joint positions. Iterations for PRN-FCN and PRN-CNN are 70,000 and 120,000, respectively. For CNN architectures, 3D shapes of human bodies are directly estimated from RGB images. Image frames in the dataset are cropped using ground truth bounding box information so that a person is centered in the cropped image, which are then resized to $256 \times 256$.

For the 300-VW dataset, both 2D inputs and 3D outputs have 68 keypoints. Since the dataset is smaller than Human 3.6M, iterations for FCN and CNN are set to 14,000 and 100,000, respectively. The RGB images are also cropped so that the mean position of the 2D landmarks becomes the center of the image.

We used Adam optimizer~\citeNew{kingma2014adam} with the start learning rate of $10^{-4}$. The learning rate is decayed by 0.8 for every 5,000 iterations. Number of hidden nodes in fully-connected ResBlock is set to 1,024 for Human 3.6M and 300-VW datasets and to 4,096 for SURREAL dataset respectively.

\section{Additional Experimental Results}

\subsection{Robustness to missing points on Human 3.6M dataset}
We measured the robustness of PRN-FCN when there exist missing points in both training and test datasets. We increased the ratio of missing points from 0\% to 20\% with 5\% intervals and trained PRN for each case on Human 3.6M dataset. The missing points are randomly selected and 2D inputs of missing points are set to 0. The MPJPE of PRN-FCN for all cases are shown in Table~\ref{tab1_supp}. It is verified that PRN is robust to missing points since the error only slightly increases as the ratio of missing point gets larger. 

\begin{table*}[t]
\caption{MPJPE with missing 2D inputs on Human 3.6M dataset.}
\centering
\label{tab1_supp}         
\begin{tabular}{lccccc}
\hline
Missing Points Ratio (\%) \qquad & 0 & 5 & 10 & 15 & 20\\
\hline
PRN-FCN & \,66.7\, & \,69.8\, & \,70.8\, & \,72.1\, & \,73.8\,
 \\
\hline
\end{tabular}
\end{table*}

\begin{figure}[t]
\centering
\begin{minipage}[r]{0.18\textwidth}
\centering
Image \\ inputs
\end{minipage}
\begin{minipage}[r]{0.18\textwidth}
\centering
\includegraphics[width=0.8\textwidth]{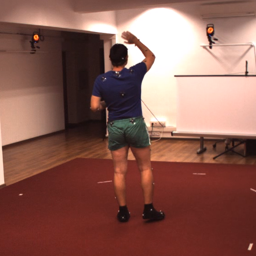}
\end{minipage}
\begin{minipage}[r]{0.18\textwidth}
\centering
\includegraphics[width=0.8\textwidth]{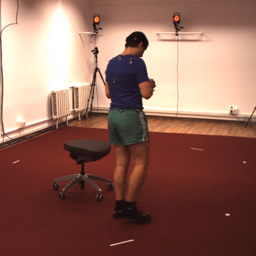}
\end{minipage}
\begin{minipage}[r]{0.18\textwidth}
\centering
\includegraphics[width=0.8\textwidth]{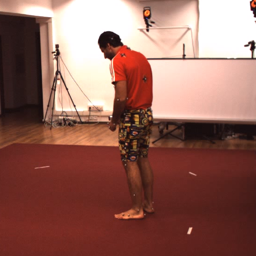}
\end{minipage}
\begin{minipage}[r]{0.18\textwidth}
\centering
\includegraphics[width=0.8\textwidth]{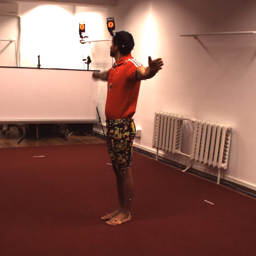}
\end{minipage}

\begin{minipage}[r]{0.18\textwidth}
\centering
PRN-CNN
\end{minipage}
\begin{minipage}[r]{0.18\textwidth}
\centering
\includegraphics[width=0.7\textwidth]{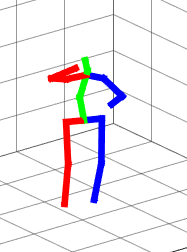}
\end{minipage}
\begin{minipage}[r]{0.18\textwidth}
\centering
\includegraphics[width=0.7\textwidth]{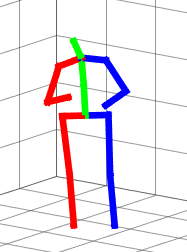}
\end{minipage}
\begin{minipage}[r]{0.18\textwidth}
\centering
\includegraphics[width=0.7\textwidth]{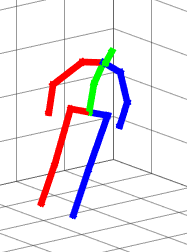}
\end{minipage}
\begin{minipage}[r]{0.18\textwidth}
\centering
\includegraphics[width=0.7\textwidth]{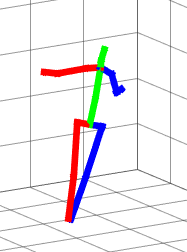}
\end{minipage}
\begin{minipage}[r]{0.18\textwidth}
\centering
GT
\end{minipage}
\begin{minipage}[r]{0.18\textwidth}
\centering
\includegraphics[width=0.7\textwidth]{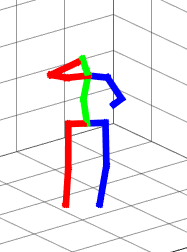}
\end{minipage}
\begin{minipage}[r]{0.18\textwidth}
\centering
\includegraphics[width=0.7\textwidth]{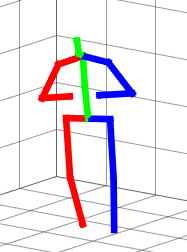}
\end{minipage}
\begin{minipage}[r]{0.18\textwidth}
\centering
\includegraphics[width=0.7\textwidth]{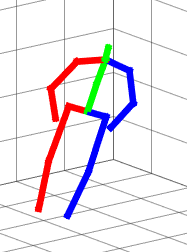}
\end{minipage}
\begin{minipage}[r]{0.18\textwidth}
\centering
\includegraphics[width=0.7\textwidth]{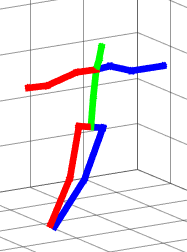}
\end{minipage}
  \caption{Qualitative results of PRN-CNN with RGB image inputs on Human 3.6M dataset. The last column shows a failure case.}
  \label{fig1_supp}
\end{figure}

\subsection{Qualitative results of PRN-CNN on Human 3.6M dataset}
Qualitative results of PRN-CNN are provided in Figure~\ref{fig1_supp} with input images and ground truths. PRN-CNN is able to estimate accurate 3D poses for various images including the case when \ml{a} part of body joints are self-occluded. The last column of Figure~\ref{fig1} shows a common failure case. PRN outputs poor estimation results when the input images are side-view\ml{s} of human bodies and the ground truth 3D pose\ml{s} have large depth \ml{variations}. For those cases, \nj{a} large depth change may only \nj{lead} to small changes in RGB images, so PRN suffers from distinguishing those subtle changes.

\begin{figure}[t]
\centering
\begin{minipage}[r]{0.18\textwidth}
\centering
Image inputs
\end{minipage}
\begin{minipage}[r]{0.18\textwidth}
\centering
\includegraphics[width=0.9\textwidth]{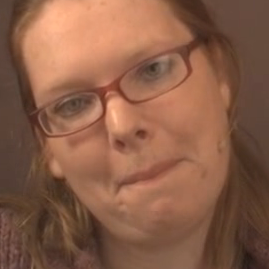}
\end{minipage}
\begin{minipage}[r]{0.18\textwidth}
\centering
\includegraphics[width=0.9\textwidth]{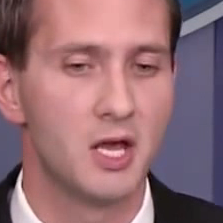}
\end{minipage}
\begin{minipage}[r]{0.18\textwidth}
\centering
\includegraphics[width=0.9\textwidth]{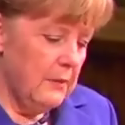}
\end{minipage}
\begin{minipage}[r]{0.18\textwidth}
\centering
\includegraphics[width=0.9\textwidth]{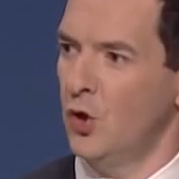}
\end{minipage}
\begin{minipage}[r]{0.18\textwidth}
\centering
PRN-CNN \\ (XY view)
\end{minipage}
\begin{minipage}[r]{0.18\textwidth}
\centering
\includegraphics[width=0.9\textwidth]{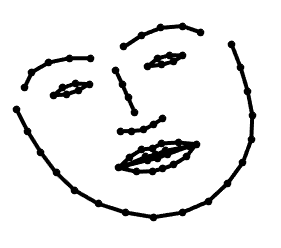}
\end{minipage}
\begin{minipage}[r]{0.18\textwidth}
\centering
\includegraphics[width=0.9\textwidth]{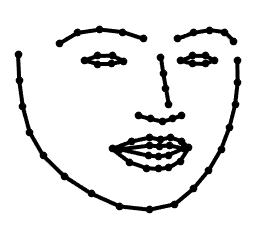}
\end{minipage}
\begin{minipage}[r]{0.18\textwidth}
\centering
\includegraphics[width=0.9\textwidth]{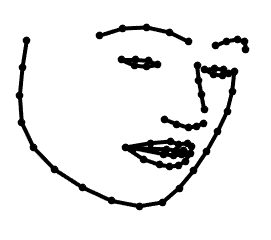}
\end{minipage}
\begin{minipage}[r]{0.18\textwidth}
\centering
\includegraphics[width=0.9\textwidth]{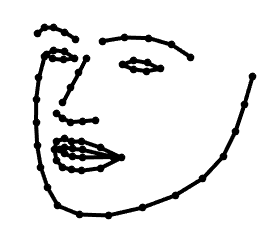}
\end{minipage}
\begin{minipage}[r]{0.18\textwidth}
\centering
PRN-CNN \\ (YZ view)
\end{minipage}
\begin{minipage}[r]{0.18\textwidth}
\centering
\includegraphics[width=0.9\textwidth]{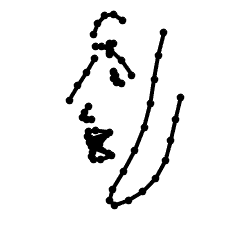}
\end{minipage}
\begin{minipage}[r]{0.18\textwidth}
\centering
\includegraphics[width=0.9\textwidth]{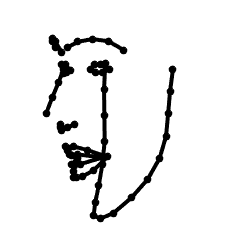}
\end{minipage}
\begin{minipage}[r]{0.18\textwidth}
\centering
\includegraphics[width=0.9\textwidth]{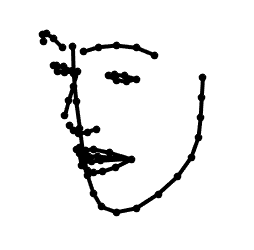}
\end{minipage}
\begin{minipage}[r]{0.18\textwidth}
\centering
\includegraphics[width=0.9\textwidth]{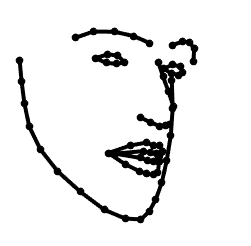}
\end{minipage}

  \caption{Qualitative results of PRN-CNN for 3D face shape estimation.}
  \label{fig2_supp}
\end{figure}

\subsection{Performance on 300-VW dataset with RGB inputs}
We also provided the quantitative results on 300-VW dataset trained with RGB image inputs. As illustrated in Table~\ref{tab2_supp}, PRN-CNN performs better than the compared CNN models which are trained using NRSfM results.

We have also shown a few reconstruction results of PRN-CNN on various test images in Figure~\ref{fig2}. The results from PRN are illustrated in views from \ml{the} XY-plane and \ml{the} YZ-plane. From \ml{the} XY-plane viewpoints, it is verified that \ml{the} 2D poses of faces are successfully learned in PRN, \sh{although it is hard to capture subtle variations of eyes or mouth. We can also verify from the YZ-plane viewpoints that the depth information is correctly estimated.}

\begin{table}[t]
\caption{Normalized error with RGB inputs on the 300-VW dataset.}
\centering
\label{tab2_supp}         
\begin{tabular}{cc}
\hline
Method & Normalized Error\\
\hline
CSF2~\citeNew{gotardo2011non_supp} + CNN & 0.4331
 \\
PR~\citeNew{PE_TIP_supp} + CNN & 0.4224
 \\
PRN-CNN & \textbf{0.3092}
 \\
\hline
\end{tabular}
\end{table}

\bibliographystyleNew{splncs04}
\bibliographyNew{egbib_supp}
\end{document}